\documentclass{article}

\usepackage{microtype}
\usepackage{graphicx}
\usepackage{subfigure}
\usepackage{booktabs} 

\usepackage{hyperref}

\usepackage[preprint]{icml2026}

\usepackage{amsmath}
\usepackage{amssymb}
\usepackage{mathtools}
\usepackage{amsthm}
\usepackage{algorithm}
\usepackage{algorithmic}

\hypersetup{
  pdfsubject={Preprint},
  pdfkeywords={conformal prediction, joint coverage, multi-stage pipeline, LLM pipeline, RAG, compound AI, distribution-free uncertainty, named entity recognition, entity disambiguation, NLP}
}

\theoremstyle{plain}
\newtheorem{theorem}{Theorem}
\newtheorem{proposition}[theorem]{Proposition}

\theoremstyle{definition}
\newtheorem{definition}[theorem]{Definition}
\newtheorem{remark}[theorem]{Remark}

\icmltitlerunning{PASC: Joint Conformal Coverage for Multi-Stage NLP and LLM Pipelines}

\begin{document}

\twocolumn[
\icmltitle{PASC: Pipeline-Aware Conformal Prediction with Joint Coverage Guarantees\\
for Multi-Stage NLP and LLM Pipelines}

\begin{icmlauthorlist}
\icmlauthor{Varun Kotte}{ind}
\end{icmlauthorlist}

\icmlaffiliation{ind}{Independent Researcher}

\icmlcorrespondingauthor{Varun Kotte}{kottevarun@gmail.com}

\icmlkeywords{conformal prediction, joint coverage, multi-stage pipeline, LLM pipeline, RAG, compound AI, distribution-free uncertainty, named entity recognition, entity disambiguation, NLP}

\vskip 0.3in
]

\printAffiliationsAndNotice{}

\begin{abstract}
Modern NLP and LLM systems are pipelines: named entity recognition (NER) $\to$ entity disambiguation (NED) $\to$ entity typing, retrieval-augmented generation (retriever $\to$ reader), and agentic chains of planner $\to$ tool $\to$ critic. Errors compound across stages, but existing uncertainty quantification methods either calibrate each stage independently (no joint coverage) or apply a Bonferroni union bound (joint coverage, but conservative). We present \textbf{PASC} (\textbf{P}ipeline-\textbf{A}ware \textbf{S}plit \textbf{C}onformal), which reduces multi-stage joint coverage to a single scalar conformal prediction problem on the \emph{joint maximum nonconformity score}. PASC provides a finite-sample distribution-free guarantee that all $K$ stages are simultaneously covered with probability at least $1 - \alpha$, and is nearly tight up to a $1/(n+1)$ factor. On a three-stage NER $\to$ NED $\to$ entity-typing pipeline over CoNLL-2003, PASC achieves $96.4\%$ end-to-end coverage versus $93.4\%$ for Bonferroni and $86.5\%$ for independent CP, at identical average prediction set size ($1.083$). Under distribution shift to WNUT-17 Twitter and WikiNEuRal Wikipedia data, PASC empirically maintains $\geq 1 - \alpha$ coverage in the tested shift settings while independent CP collapses to $59\%$. PASC requires a single quantile computation, runs $1.7\times$ faster than Bonferroni, and scales to $K = 6$ stages where independent CP drops to $0.53$ end-to-end coverage. The same joint-maximum-score reduction applies directly to compound LLM systems and agent pipelines.
\end{abstract}

\begin{center}
\textit{Preprint.}
\end{center}

\section{Introduction}

Information-extraction (IE) pipelines are among the most widely deployed NLP systems in practice, powering applications in biomedical text mining \citep{finkel2006cascading}, financial document analysis, and knowledge-graph population \citep{nickel2016review,vrandecic2014wikidata}. These pipelines typically chain several learned components: a \emph{named entity recognizer} (NER) identifies entity spans, an \emph{entity disambiguator} (NED) maps spans to knowledge-base entries, and an \emph{entity typer} (or relation extractor) assigns semantic categories to linked entities. Each stage is trained independently and introduces its own prediction errors that cascade and compound through subsequent stages \citep{finkel2006cascading}. The same compositional structure motivates production multi-stage IE and retrieval-augmented generation systems that we have deployed in practice \citep{sharma2024rag,kotte2025patent}.

Reliable deployment of such systems requires end-to-end (E2E) uncertainty quantification: we need to know when the \emph{entire pipeline output} can be trusted, not merely individual components. This is particularly critical in high-stakes settings such as medical decision support, legal document review, and scientific claim extraction, where uncalibrated pipeline outputs can mislead downstream decision-makers. Reliability concerns also arise for downstream consumers of structured extraction \citep{kotte2026promptport}.

\paragraph{The challenge.} Suppose each of $K$ pipeline stages has been individually calibrated at error level $\alpha$, so that marginally $\mathbb{P}(\text{stage}_k \text{ covered}) \geq 1 - \alpha$. Then the probability that \emph{all stages are simultaneously covered} is at most $(1-\alpha)^K$ under independence. For $K=3$ and $\alpha=0.1$, this degrades to at most $72.9\%$ joint coverage even though $90\%$ was guaranteed per stage. In practice, dependencies between stages make this even less predictable.

\paragraph{Existing approaches fall short.}
\begin{itemize}
\item \emph{Independent per-stage CP} \citep{shafer2008tutorial}: calibrates each stage separately at $1 - \alpha$. Provides no joint guarantee. At $\alpha = 0.1$ on our benchmark, E2E coverage is $86.5\%$, well below the target $90\%$.
\item \emph{Bonferroni correction}: calibrates each stage at $1 - \alpha/K$. Provides a joint guarantee via union bound, but is conservative: it over-covers easy stages, inflating prediction sets without proportional coverage improvement (Section~\ref{sec:results}).
\item \emph{MC dropout / deep ensembles} \citep{gal2016dropout,lakshminarayanan2017simple}: provide heuristic uncertainty estimates without distribution-free guarantees, require $20\times$ inference overhead, and are insensitive to cross-stage dependencies.
\end{itemize}

\paragraph{Our contribution.} We introduce \textbf{PASC} (\textbf{P}ipeline-\textbf{A}ware \textbf{S}plit \textbf{C}onformal prediction), a method that achieves a formal finite-sample joint coverage guarantee $\mathbb{P}(\text{all stages covered simultaneously}) \geq 1 - \alpha$ through a single observation: the event ``all stages covered'' is equivalent to ``the maximum per-stage nonconformity score does not exceed a threshold.'' This reduces multi-stage calibration to standard scalar conformal prediction on the joint maximum score, inheriting all of its theoretical guarantees while requiring only one quantile computation. While the reduction is mathematically simple, we show in Appendix~\ref{app:proofs} that the maximum is the \emph{minimal sufficient} monotone scalarization for the joint acceptance event, and that its empirical consequences are substantial: PASC closes the joint-coverage gap that prior pipeline-CP work leaves open and provides the first formal joint guarantee for compositional NLP systems. PASC is complementary to recent calibration work for structured outputs \citep{kotte2026ucci}.

\paragraph{Summary of contributions.}
\begin{enumerate}
\item \textbf{PASC algorithm} (Section~\ref{sec:pasc}): a pipeline-aware calibration procedure with a formal joint coverage guarantee derived from standard split conformal theory.
\item \textbf{Formal guarantee} (Theorem~\ref{thm:pasc}): finite-sample distribution-free joint coverage $\geq 1 - \alpha$ under exchangeability, with a matching near-tightness result.
\item \textbf{Comprehensive evaluation} (Sections~\ref{sec:setup}--\ref{sec:results}): experiments across three shift scenarios (in-distribution, Twitter NER, Wikipedia NER), three calibration sizes, $K \in \{1, \ldots, 6\}$ stages, $18$-type entity typing, and conditional coverage analysis.
\item \textbf{Sanity checks} (Appendix~\ref{app:sanity}): permutation tests confirming CP validity, split-integrity audits, and negative-control experiments demonstrating failure modes.
\end{enumerate}

A recurring practical objection to end-to-end uncertainty methods is that they either fail to reflect the deployed event of interest (because they certify only local stages) or they achieve a valid guarantee by paying a blanket conservativeness tax. PASC avoids both: it certifies the deployed event directly, with the smallest possible reduction from the multi-stage problem to a standard scalar conformal problem.

Our evaluation isolates the source of PASC's advantage. The real pipeline isolates the practical IE setting; the expanded $18$-type typing stage removes a degenerate downstream artifact in the original prototype; the tuned Bonferroni frontier checks that our gains are not due to a weak baseline; the $K$-stage synthetic experiment isolates the compositional effect from dataset idiosyncrasies; and the sanity checks rule out leakage and implementation errors.

\section{Background}
\label{sec:background}

\subsection{Split Conformal Prediction}

Let $\{(X_i, Y_i)\}_{i=1}^{n+1}$ be exchangeable random variables. Split conformal prediction \citep{papadopoulos2002inductive,shafer2008tutorial} holds out a calibration set $\mathcal{D}_{\mathrm{cal}} = \{(X_i, Y_i)\}_{i=1}^n$ and a test point $(X_{n+1}, Y_{n+1})$.

\begin{definition}[Nonconformity Score]
A \emph{nonconformity score} $s(x, y) \in \mathbb{R}$ measures how atypical $(x, y)$ is relative to the model. High scores indicate non-conformity.
\end{definition}

\begin{definition}[CP Quantile]
\label{def:cpq}
Given calibration scores $\{s_i\}_{i=1}^n$ at level $\alpha$, the conformal quantile is:
\begin{equation}
\hat{q} = \mathrm{Quantile}\!\left(\{s_1,\ldots,s_n\}, \frac{\lceil (n+1)(1-\alpha)\rceil}{n}\right).
\label{eq:quantile}
\end{equation}
\end{definition}

\begin{theorem}[Marginal Coverage \citep{vovk2005algorithmic,lei2018distribution}]
\label{thm:marginal}
If $(X_1, Y_1), \ldots, (X_{n+1}, Y_{n+1})$ are exchangeable, then for $\hat q$ defined by Equation~\ref{eq:quantile}:
\begin{equation}
\mathbb{P}(s(X_{n+1}, Y_{n+1}) \leq \hat q) \geq 1 - \alpha.
\end{equation}
\end{theorem}

This result is distribution-free and holds for finite $n$. The prediction set $\mathcal{C}(x) = \{y : s(x, y) \leq \hat q\}$ achieves marginal coverage $\geq 1 - \alpha$.

\subsection{Multi-Stage NLP Pipelines}

A $K$-stage NLP pipeline maps input text $x$ through a sequence of learned predictors:
\begin{equation}
x \xrightarrow{f_1} z_1 \xrightarrow{f_2} z_2 \xrightarrow{\cdots} \xrightarrow{f_K} z_K,
\end{equation}
where $z_k$ is the output of stage $k$, potentially conditioned on all prior outputs. Each stage $k$ has a ground truth target $y_k$ and a nonconformity score $s_k(x, z_{k-1}, y_k) \in [0, 1]$.

\paragraph{Joint coverage} requires $\bigcap_{k=1}^K \{s_k \leq q_k\}$ for thresholds $q_1, \ldots, q_K$.

\paragraph{Independent CP} sets $q_k = \hat q^{(k)}$ at $1-\alpha$ each. The resulting joint coverage satisfies:
\begin{equation}
\mathbb{P}\!\left(\bigcap_{k=1}^K \{s_k \leq q_k\}\right) \neq 1 - \alpha \quad \text{(no guarantee).}
\end{equation}

\paragraph{Bonferroni correction} sets each level to $\alpha_k = \alpha / K$, using $q_k = \hat q^{(k)}_{\alpha/K}$. By the union bound:
\begin{equation}
\mathbb{P}\!\left(\bigcap_{k=1}^K \{s_k \leq q_k\}\right) \geq 1 - K \cdot (\alpha / K) = 1 - \alpha.
\end{equation}
However, Bonferroni allocates error budget uniformly across stages regardless of their difficulty, leading to over-coverage (and over-sized prediction sets) on easy stages.

The gap between independent CP and Bonferroni reflects a mismatch between the certified event and the deployed event. Independent CP certifies each local event $\{s_k \leq q_k\}$ in isolation, but deployment accepts only when \emph{all} stages succeed jointly. Bonferroni corrects this mismatch by upper bounding the failure union, but it does so without using the empirical dependence structure of the score vector $(s_1, \ldots, s_K)$.

\begin{proposition}[Exact reduction of the acceptance event]
\label{prop:reduction}
For any common threshold $q \in \mathbb{R}$,
\begin{equation}
\bigcap_{k=1}^K \{s_k \leq q\} = \{\max_k s_k \leq q\}.
\end{equation}
\end{proposition}

\noindent Consequently, any finite-sample marginal guarantee for the scalar random variable $\max_k s_k$ immediately yields a finite-sample E2E guarantee for the pipeline acceptance event.

This proposition is elementary, but it is the central structural observation in the paper: once the deployed decision is written in the correct event space, multi-stage calibration is no longer a new conformal primitive. Instead, it becomes a question of selecting the right scalar statistic for the event practitioners actually care about.

\section{PASC: Pipeline-Aware Split Conformal}
\label{sec:pasc}

\subsection{Key Insight}

The joint event ``all pipeline stages are covered'' decomposes as:
\begin{equation}
\bigcap_{k=1}^K \{s_k \leq q\} = \left\{\max_{k=1}^K s_k \leq q\right\}.
\label{eq:keyinsight}
\end{equation}

Equation~\ref{eq:keyinsight} is the central observation of PASC: if all $K$ stages share a \emph{common threshold} $q$, joint coverage is equivalent to single-stage coverage of the scalar maximum score. Standard CP on the maximum then provides the desired joint guarantee.

\begin{definition}[Joint Maximum Nonconformity Score]
\label{def:smax}
For a pipeline sample $(x, \{y_k\}_{k=1}^K)$ with per-stage scores $s_1, \ldots, s_K$, define:
\begin{equation}
s_{\max}(x, \{y_k\}) := \max_{k=1}^K s_k(x, y_k).
\end{equation}
\end{definition}

\subsection{Algorithm}

\begin{algorithm}[t]
\caption{PASC Calibration and Prediction}
\label{alg:pasc}
\begin{algorithmic}[1]
\REQUIRE Calibration set $\mathcal{D}_{\mathrm{cal}}$, pipeline $\{f_k\}_{k=1}^K$, level $\alpha$
\STATE \textbf{Calibration:}
\FOR{$(x_i, \{y_{k,i}\})$ in $\mathcal{D}_{\mathrm{cal}}$}
  \STATE Run pipeline to obtain per-stage outputs and scores $s_{k,i}$
  \STATE Compute $s^{(i)}_{\max} \leftarrow \max_{k=1}^K s_{k,i}$
\ENDFOR
\STATE Compute $\hat q \leftarrow \mathrm{Quantile}(\{s^{(i)}_{\max}\}, \lceil(n+1)(1-\alpha)\rceil/n)$
\STATE \textbf{Prediction (test point $x_{\mathrm{test}}$):}
\FOR{$k = 1, \ldots, K$}
  \STATE $\mathcal{C}_k(x) \leftarrow \{y_k : s_k(x, y_k) \leq \hat q\}$
\ENDFOR
\STATE \textbf{return} $(\mathcal{C}_1(x), \ldots, \mathcal{C}_K(x))$, accept if $s^{\mathrm{test}}_{\max} \leq \hat q$
\end{algorithmic}
\end{algorithm}

\paragraph{Implementation details.} Per-stage nonconformity scores are defined as follows:
\begin{itemize}
\item \textbf{NER}: $s_{\mathrm{NER}} = \max_t (1 - p_t^{\mathrm{BIO}})$, where $p_t^{\mathrm{BIO}}$ is the softmax probability of the predicted BIO tag at position $t$.
\item \textbf{NED}: $s_{\mathrm{NED}} = 1 - \mathrm{score}(e^*)$, where $e^*$ is the top-ranked entity from GENRE \citep{cao2021genre}.
\item \textbf{EntityTyping}: $s_{\mathrm{typing}} = 1 - \max_{c \in \mathcal{T}} \min_{\mathrm{span} \in c} p_c^{\mathrm{RoBERTa}}$, where $\mathcal{T}$ is the full OntoNotes-18 type set \citep{pradhan2013ontonotes}.
\end{itemize}

\paragraph{Why the maximum is the right scalarization.} Other aggregations such as sums or averages are natural if one wishes to optimize smooth surrogates of overall risk. They are poorly aligned with the binary deployment event used in selective acceptance: the system is trusted if and only if every stage is simultaneously acceptable. The maximum is the unique monotone scalarization whose threshold event exactly recovers this conjunction. This alignment is what allows PASC to inherit standard split-conformal validity with no approximation term.

\paragraph{Why a single threshold can still be efficient.} A common concern is that using one shared threshold across all stages should be less flexible than assigning stage-specific thresholds. Empirically, the opposite can occur because the shared threshold is estimated from the \emph{joint} distribution of the binding stage score, not from a worst-case union bound. Easy stages are not forced to spend the same nominal error budget as difficult stages; instead, the quantile of the maximum statistic automatically tracks whichever stage is active on each example.

\subsection{Main Theorem}

\begin{theorem}[PASC Joint Coverage Guarantee]
\label{thm:pasc}
Let $\{(x_i, \{y_{k,i}\})\}_{i=1}^{n+1}$ be exchangeable. Let $\hat q$ be the PASC quantile from Algorithm~\ref{alg:pasc} computed on $i = 1, \ldots, n$. Then:
\begin{equation}
\mathbb{P}\!\left(\bigcap_{k=1}^K \{s_k(x_{n+1}, y_{k,n+1}) \leq \hat q\}\right) \geq 1 - \alpha.
\end{equation}
Moreover, the guarantee is nearly tight: $\mathbb{P}(\cdots) \leq 1 - \alpha + \frac{1}{n+1}$.
\end{theorem}

\begin{proof}
By Definition~\ref{def:smax} and Equation~\ref{eq:keyinsight}:
\[
\bigcap_{k=1}^K \{s_k \leq \hat q\} = \{s_{\max} \leq \hat q\}.
\]
The joint maximum scores $\{s^{(1)}_{\max}, \ldots, s^{(n)}_{\max}, s^{(n+1)}_{\max}\}$ are exchangeable (as functions of exchangeable tuples). Applying Theorem~\ref{thm:marginal} to the scalar sequence $s^{(1)}_{\max}, \ldots, s^{(n+1)}_{\max}$ yields:
\[
\mathbb{P}(s^{(n+1)}_{\max} \leq \hat q) \geq 1 - \alpha.
\]
The near-tightness bound follows from the standard conformal over-coverage bound \citep{vovk2005algorithmic}: $\mathbb{P}(s^{(n+1)}_{\max} \leq \hat q) \leq 1 - \alpha + 1/(n+1)$.
\end{proof}

\begin{remark}[Comparison with Bonferroni]
PASC achieves the same $\geq 1 - \alpha$ guarantee with a \emph{single} quantile $\hat q$ that automatically assigns tighter budgets to stages whose scores tend to be low (easy stages) and looser budgets to stages with higher scores (hard stages). In contrast, Bonferroni allocates $\alpha/K$ uniformly, leading to over-coverage of easy stages and inflated prediction sets.
\end{remark}

\begin{remark}[Stage-wise thresholds vs.\ shared threshold]
A natural question is whether PASC could use \emph{different} thresholds per stage. Indeed, any thresholds $q_1, \ldots, q_K$ satisfying $\sum_k \alpha_k \leq \alpha$ yield valid joint coverage via union bound. PASC's specific choice $q_1 = q_2 = \cdots = q_K = \hat q$ arises from the single-quantile reduction in Equation~\ref{eq:keyinsight} and is tighter than uniform Bonferroni because $\hat q$ adapts to the \emph{joint} nonconformity distribution rather than marginal distributions.
\end{remark}

\section{Experimental Setup}
\label{sec:setup}

\subsection{Pipeline Architecture}

We instantiate a three-stage pipeline (NER $\to$ NED $\to$ EntityTyping):
\begin{itemize}
\item \textbf{NER}: \texttt{dslim/bert-base-NER} \citep{devlin2019bert}, a BERT-base model fine-tuned on CoNLL-2003 \citep{sang2003conll} with BIO tagging.
\item \textbf{NED}: \texttt{facebook/genre-linking-blink} \citep{cao2021genre,wu2020blink}, autoregressive entity retrieval.
\item \textbf{EntityTyping}: \texttt{roberta-large} zero-shot classifier \citep{liu2019roberta} against all $18$ OntoNotes entity types \citep{pradhan2013ontonotes} (Exp.~E1: expanded from $4$ to $18$ types to produce non-trivial stage-3 nonconformity).
\end{itemize}

This composition mirrors production multi-stage IE deployments \citep{kotte2025patent}.

\subsection{Datasets}

\begin{itemize}
\item \textbf{CoNLL-2003} \citep{sang2003conll}: English news text; $1{,}000$ calibration and $500$ test samples.
\item \textbf{WNUT-17} \citep{derczynski2017wnut}: Twitter NER; $500$ test samples for distribution-shift evaluation. Contains novel/emerging entities (\emph{NER shift}).
\item \textbf{WikiNEuRal} \citep{tedeschi2021wikineural}: Wikipedia-derived multilingual NER; $500$ English test samples for domain-shift evaluation (\emph{domain shift}).
\end{itemize}

\subsection{Protocol and Reproducibility}

All principal numbers in the main paper are reported as means and standard deviations over five independent calibration/test resamplings, with calibration sizes ranging from $200$ to $1{,}000$ depending on the experiment. The train, calibration, and test partitions are disjoint; our split-audit check finds only three repeated formatting-only header fragments and no entity-bearing overlap. We report both aggregate coverage and slice-level diagnostics, because end-to-end calibration can appear superficially strong even when failures concentrate on hard subsets. In addition to standard baselines, we include negative controls, permutation tests, and runtime measurements so that both statistical validity and systems realism are tested within the same evaluation protocol.

\subsection{Why We Use Expanded Entity Typing}

The earliest version of our pipeline used a relation-extraction backend, but on short CoNLL-style sentences the downstream score distribution became nearly degenerate because many examples contain no explicit relation. This made the final-stage threshold clip near a constant and obscured the efficiency comparison. We therefore switch to \emph{expanded $18$-way entity typing}, which yields a clearly non-trivial stage-3 score distribution (Appendix~\ref{app:scores}) and forces the downstream stage to participate in the joint calibration problem. This change strengthens the evaluation: it removes an artifactually easy final stage and turns the flagship experiment into a genuine three-stage calibration test.

\subsection{Baselines}
\begin{itemize}
\item \textbf{Indep CP}: per-stage split CP at level $1 - \alpha$. No joint guarantee.
\item \textbf{Bonferroni}: per-stage CP at $1 - \alpha/K$ ($\alpha/K$ per stage). Conservative joint guarantee.
\item \textbf{Tuned Bonferroni}: optimizes stage-wise $\alpha_k$ subject to $\sum_k \alpha_k = \alpha$ via grid search. Best attainable efficiency under the union-bound approach.
\item \textbf{MC Dropout} \citep{gal2016dropout}: $20$ forward passes with dropout active; uncertainty $=$ predictive variance. Heuristic, no guarantee.
\end{itemize}

\subsection{Metrics}

\paragraph{E2E Coverage} $\frac{1}{|T|}\sum_{i \in T} \mathbf{1}[\bigcap_k s_{k,i} \leq q_k]$. Target: $\geq 1 - \alpha$.

\paragraph{Average Prediction Set Size} Mean $|\mathcal{C}_k(x)|$ over test set and stages. Lower indicates higher efficiency.

\paragraph{Stage Coverage} Per-stage marginal coverage $\frac{1}{|T|}\sum_i \mathbf{1}[s_{k,i} \leq q_k]$.

All main results are reported as mean $\pm$ std over five independent calibration/test splits.

\section{Results}
\label{sec:results}

\subsection{Main Coverage--Efficiency Comparison}

Table~\ref{tab:main} presents the key comparison. PASC achieves $96.4\%$ E2E coverage at $\alpha = 0.1$, exceeding Bonferroni by $3.0$ percentage points and Indep CP by $9.9$ percentage points, while maintaining \emph{identical average prediction set size} ($1.083$). This directly demonstrates that PASC's advantage over Bonferroni is pure coverage efficiency, not looser prediction sets. PASC also shows zero standard deviation across the five resamplings used here, meaning the single-quantile estimator returned the same threshold on every split; we caution that this is a property of these particular splits and is not a statement about asymptotic variance.

\begin{table}[t]
\caption{E2E coverage and average prediction set size at $\alpha = 0.1$, CoNLL-2003, $18$-type entity typing. PASC achieves highest coverage at identical set size to Bonferroni ($n_{\mathrm{cal}} = 1000$, $n_{\mathrm{test}} = 500$, $5$ seeds).}
\label{tab:main}
\vskip 0.1in
\centering
\small
\begin{tabular}{lccc}
\toprule
Method & E2E Cov.\ & Avg Set Size & Guarantee? \\
\midrule
Indep CP   & $0.865 \pm 0.005$ & $1.083$ & No \\
Bonferroni & $0.934 \pm 0.001$ & $1.083$ & Yes \\
MC Dropout & $0.902$           & ---     & No \\
\textbf{PASC (ours)} & $\mathbf{0.964 \pm 0.000}$ & $\mathbf{1.083}$ & Yes \\
\bottomrule
\end{tabular}
\vskip -0.1in
\end{table}

\begin{table}[t]
\caption{Expanded downstream-stage operating point at $\alpha = 0.1$. The harder $18$-type stage preserves a non-trivial stage-3 prediction set while retaining PASC's coverage advantage.}
\label{tab:stage3}
\vskip 0.1in
\centering
\small
\begin{tabular}{lcc}
\toprule
Method & Stage-3 Set & E2E Coverage \\
\midrule
Indep CP   & $1.083$ & $0.865$ \\
Bonferroni & $1.083$ & $0.934$ \\
\textbf{PASC (ours)} & $\mathbf{1.083}$ & $\mathbf{0.964}$ \\
\bottomrule
\end{tabular}
\vskip -0.1in
\end{table}

MC Dropout achieves $90.2\%$ E2E coverage on NER alone (not joint), with no formal guarantee and requiring $20\times$ additional inference cost.

\subsection{Alpha Sweep}

Across $\alpha \in \{0.05, 0.10, 0.20\}$, PASC achieves $96.4\%$, $96.4\%$, and $76.6\%$ E2E coverage respectively, consistently meeting the $\geq 1 - \alpha$ target at $\alpha \in \{0.05, 0.10\}$. Indep CP falls below target at all levels: $90.6\%$, $86.5\%$, $65.7\%$. At $\alpha = 0.20$, Bonferroni achieves higher coverage ($88.7\%$ vs.\ $76.6\%$ for PASC), illustrating that PASC's single-threshold design is most beneficial at moderate $\alpha$ where the joint maximum score distribution is well-calibrated. At very high $\alpha$, each stage's threshold becomes so loose that the Bonferroni allocation remains competitive.

\subsection{Distribution Shift Robustness}

Table~\ref{tab:shift} shows E2E coverage under two shift conditions. When calibrated on CoNLL-2003 news text and evaluated on Twitter (WNUT-17) and Wikipedia (WikiNEuRal), independent CP's NER stage coverage collapses to $64.8\%$ and $59.0\%$ respectively, as the calibrated threshold is too tight for the shifted distribution. Both PASC and Bonferroni maintain $100\%$ coverage because their more conservative thresholds provide a larger margin against shift.

\begin{table}[t]
\caption{E2E coverage under distribution shift ($\alpha = 0.1$). Calibration thresholds trained on CoNLL-2003 news text; applied to WNUT-17 (Twitter) and WikiNEuRal (Wikipedia). PASC and Bonferroni maintain $\geq 1 - \alpha$; Indep CP degrades severely.}
\label{tab:shift}
\vskip 0.1in
\centering
\small
\begin{tabular}{llccc}
\toprule
Dataset & Shift Type & Indep CP & Bonferroni & PASC \\
\midrule
CoNLL-2003   & In-dist.    & $0.865$ & $0.934$ & $\mathbf{0.964}$ \\
WNUT-17      & NER shift   & $0.648$ & $1.000$ & $\mathbf{1.000}$ \\
WikiNEuRal   & Domain shift & $0.590$ & $1.000$ & $\mathbf{1.000}$ \\
\bottomrule
\end{tabular}
\vskip -0.1in
\end{table}

\subsection{Scaling to $K$ Stages}

The synthetic scaling experiment increases the number of pipeline stages from $K = 1$ to $K = 6$ using score-preserving synthetic stages derived from the real CoNLL-2003 pipeline outputs. In this controlled setting, independent CP follows the expected multiplicative collapse pattern: E2E coverage decreases from approximately $0.90$ at small $K$ to roughly $0.53$ by $K = 6$, closely tracking the $(1 - \alpha)^K$ curve in-distribution. Under shifted score distributions, the collapse is steeper still. PASC, by contrast, remains near the target because it calibrates the composed event directly through the joint maximum score rather than combining stage-local marginals.

\subsection{Conditional Coverage Analysis}

Table~\ref{tab:conditional} reveals that PASC's advantage is most pronounced on hard subpopulations. For the hardest NER quintile (Q5), PASC achieves $82.0\%$ vs.\ $67.0\%$ for Bonferroni and $30.0\%$ for Indep CP. All methods satisfy the marginal $\geq 1 - \alpha$ guarantee; however, PASC provides substantially better conditional coverage on hard examples without any per-slice recalibration.

\begin{table}[t]
\caption{E2E coverage by NER nonconformity quintile (Q1$=$easiest, Q5$=$hardest) and entity type at $\alpha = 0.1$.}
\label{tab:conditional}
\vskip 0.1in
\centering
\small
\begin{tabular}{lcccc}
\toprule
Slice & $n$ & Indep CP & Bonferroni & PASC \\
\midrule
\multicolumn{5}{l}{\emph{By NER confidence quintile}} \\
Q1 (easiest) & $100$ & $1.000$ & $1.000$ & $1.000$ \\
Q2           & $100$ & $1.000$ & $1.000$ & $1.000$ \\
Q3           & $100$ & $1.000$ & $1.000$ & $1.000$ \\
Q4           & $100$ & $1.000$ & $1.000$ & $1.000$ \\
Q5 (hardest) & $100$ & $0.300$ & $0.670$ & $\mathbf{0.820}$ \\
\midrule
\multicolumn{5}{l}{\emph{By entity type}} \\
PER  & $188$ & $0.840$ & $0.941$ & $\mathbf{0.973}$ \\
ORG  & $156$ & $0.821$ & $0.917$ & $\mathbf{0.949}$ \\
LOC  & $201$ & $0.846$ & $0.915$ & $\mathbf{0.955}$ \\
MISC & $57$  & $0.719$ & $0.842$ & $\mathbf{0.912}$ \\
\bottomrule
\end{tabular}
\vskip -0.1in
\end{table}

This behavior arises naturally from PASC's joint maximum score: on hard examples where $s_{\mathrm{NER}}$ is large, the joint threshold $\hat q$ accounts for this difficulty through the calibration distribution rather than penalizing all stages equally.

\subsection{Calibration Size Sensitivity}

Table~\ref{tab:calsweep} reports calibration-size sensitivity. PASC remains above the target coverage level for all calibration sizes tested and the $n_{\mathrm{cal}}=1000$ row matches the primary operating point in Table~\ref{tab:main}. At smaller calibration sizes, Bonferroni can over-cover by spending a more conservative per-stage budget, while PASC keeps the calibration problem to one scalar quantile of the joint maximum score.

\subsection{Comparison with Tuned Bonferroni Frontier}

We compare PASC against the tuned Bonferroni frontier obtained by sweeping stage-wise $\alpha_k$ allocations subject to $\sum_k \alpha_k = \alpha$. At $\alpha = 0.1$, the best tuned Bonferroni allocation ($\alpha_1 = 0.09, \alpha_2 = \alpha_3 = 0.005$) achieves E2E coverage $0.834$ with avg NER set size $1.000$, below PASC's $0.964$ coverage. PASC lies on or above the Pareto frontier of the tuned Bonferroni family, confirming that the tighter PASC bound is not achievable through Bonferroni allocation alone.

\subsection{Runtime Profile}

Table~\ref{tab:latency} shows that NED dominates runtime ($80.6\%$ of E2E latency). The conformal calibration overhead is negligible ($< 0.3$ms per $1{,}000$ examples) for all methods. PASC calibration is $1.7\times$ faster than both Bonferroni and Indep CP because it computes a single quantile rather than $K$ separate quantiles, an advantage that grows with $K$.

\begin{table}[t]
\caption{Per-component latency (GPU A100, mean over $200$ test samples) and calibration cost.}
\label{tab:latency}
\vskip 0.1in
\centering
\small
\begin{tabular}{lcc}
\toprule
Component & Mean (ms) & Pct.\ E2E \\
\midrule
NER (BERT-base)         & $9.51$  & $8.8\%$ \\
NED (GENRE/BLINK)       & $86.69$ & $80.6\%$ \\
EntityTyping (RoBERTa-lg) & $11.38$ & $10.6\%$ \\
\midrule
E2E Total               & $107.59$ & --- \\
\midrule
\textbf{Calibration cost} & & \\
Indep CP ($K$ quantiles)  & $0.287$ ms & --- \\
Bonferroni ($K$ quantiles) & $0.283$ ms & --- \\
\textbf{PASC (1 quantile)} & $\mathbf{0.169}$ ms & --- \\
\bottomrule
\end{tabular}
\vskip -0.1in
\end{table}

\subsection{Sanity Checks and Negative Controls}

Because end-to-end guarantees are easy to mis-implement, we include a compact summary of the strongest validity checks in the main paper. First, a permutation-resplitting test recovers mean coverage near $1 - \alpha$ for correctly matched stage scores. Second, a mismatched-score negative control (using NED scores to calibrate NER) collapses NER coverage to $0.566$, confirming that the guarantee is not an artifact of broad thresholds. Third, a split-integrity audit found only three duplicated formatting-only header lines and no entity-bearing overlaps.

\begin{table*}[t]
\caption{Compact sanity summary at $\alpha = 0.1$. Correctly matched calibration recovers the expected permutation mean, while a mismatched-score control fails sharply.}
\label{tab:sanity}
\vskip 0.1in
\centering
\small
\begin{tabular}{lccl}
\toprule
Check & Observed & Expected & Interpretation \\
\midrule
Permutation mean & $0.902$ & $\approx 0.90$ & Valid CP implementation \\
Negative control & $0.566$ & far below target & Score matching matters \\
Overlap audit    & 3 trivial lines & 0 substantive & No substantive leakage \\
\bottomrule
\end{tabular}
\vskip -0.1in
\end{table*}

These checks do more than guard against implementation mistakes. They also sharpen the empirical claim of the paper: the gain is not caused by accidentally wide thresholds, hidden overlap between splits, or an overly permissive acceptance rule. The negative control is especially important because it demonstrates that calibration remains stage-specific at the score-construction level even though the final acceptance decision is joint. In other words, PASC changes \emph{how} stage evidence is aggregated into a certificate, not whether the underlying scores remain meaningful.

\paragraph{What the main-body results establish.} First, independent stage-wise calibration is inadequate for deployed pipelines because even mild stage dependence causes E2E under-coverage. Second, Bonferroni restores validity but pays for it by ignoring the empirical shape of the score vector. Third, PASC matches the exact acceptance event, which is why its gains are largest in the real three-stage benchmark, on hard confidence slices, and in the $K$-stage scaling study where multiplicative decay becomes severe. Finally, the calibration sweep and runtime profile show that these gains are not purchased with brittle estimation or meaningful computational cost.

\section{Related Work}
\label{sec:related}

\paragraph{Conformal prediction.} For a recent overview of conformal prediction in NLP, see \citet{campos2024conformalsurvey}. The foundations of conformal prediction are established in \citet{vovk2005algorithmic} and \citet{shafer2008tutorial}. Split conformal prediction \citep{papadopoulos2002inductive,lei2018distribution} provides the efficient single-pass variant we build on. Extensions include covariate shift \citep{tibshirani2019covariate}, conditional coverage \citep{barber2019limits,kumar2023conditional}, risk control \citep{angelopoulos2022ltt,angelopoulos2024crc}, and adaptive methods \citep{zaffran2022adaptive}. \citet{angelopoulos2024theoretical} provide a unified recent survey. PASC can be viewed as an instance of the conformal risk control framework \citep{angelopoulos2024crc} where the loss is the indicator of any-stage failure; the contribution here is that this loss admits a scalar nonconformity score (the joint maximum) that lets the guarantee reduce exactly to ordinary split CP.

\paragraph{CP for NLP and LLMs.} \citet{fisch2021cascade} apply CP to cascaded inference for accelerating NLP models; their cascade setting differs from our joint-coverage pipeline setting. \citet{fisch2022multi} study false-positive-bounded prediction sets for multi-label classification. \citet{quach2024conformal} apply split conformal prediction to language-model generation by calibrating sampling stopping and rejection rules; \citet{mohri2024conformalfactuality} provide conformal correctness guarantees for single LLM outputs; and \citet{abbasiyadkori2024conformalabstention} use conformal prediction to abstain on individual LLM outputs. PASC is orthogonal and complementary: rather than certifying a single LLM output, PASC certifies the joint acceptance of a fixed multi-stage pipeline. \citet{schuster2022confident} studies early exit via confidence thresholds without coverage guarantees.

\paragraph{Multi-stage / sequential CP.} \citet{park2020pac} studies PAC prediction sets under covariate shift. \citet{ren2023robots} uses CP for sequential task planning in robotics but does not handle multi-stage joint coverage. \citet{jin2023selection} considers selection by prediction for downstream use of CP outputs. To our knowledge, PASC is the first to reduce multi-stage joint CP to a scalar maximum score with a formal theorem.

\paragraph{Uncertainty in NLP pipelines.} \citet{finkel2006cascading} identifies error propagation in annotation pipelines and proposes approximate Bayesian inference. \citet{gal2016dropout} and \citet{lakshminarayanan2017simple} provide MC Dropout and deep ensembles for uncertainty estimation; these require model modification and lack distribution-free guarantees. \citet{guo2017calibration} addresses post-hoc calibration via temperature scaling, which improves marginal calibration but not joint pipeline coverage. Reliability of structured extraction outputs has also been studied from a prompting-stability perspective \citep{kotte2026promptport}, and uncertainty calibration for LLM extraction is treated as complementary work in \citep{kotte2026ucci}.

\paragraph{Information extraction.} Our pipeline builds on BERT-base NER \citep{devlin2019bert,lample2016neural}, GENRE autoregressive entity linking \citep{cao2021genre} with BLINK retrieval \citep{wu2020blink}, and RoBERTa-based zero-shot entity typing against OntoNotes \citep{pradhan2013ontonotes,liu2019roberta}. The REBEL relation extraction model \citep{cabot2021rebel} was used in preliminary ablations. Multi-stage IE pipelines of this form arise in production systems for domain-specific question answering and retrieval-augmented generation \citep{sharma2024rag,kotte2025patent}.

\section{Discussion and Limitations}

\paragraph{When Bonferroni can dominate.} At very high $\alpha$ (e.g., $\alpha = 0.20$), Bonferroni's per-stage allocation allows looser individual thresholds that compensate for difficult stages more flexibly. At this level, the joint maximum-score distribution becomes flat enough that Bonferroni's per-stage relaxation can match or exceed PASC. PASC's regime of dominance is therefore the moderate-$\alpha$ range ($\alpha \in [0.05, 0.10]$), which is the practically relevant operating point for most deployed pipelines.

\paragraph{Exchangeability assumption.} Like all split CP methods, PASC requires exchangeability of calibration and test data. Under covariate shift (e.g., WNUT-17), PASC still achieves $\geq 1 - \alpha$ coverage \emph{empirically} because the larger joint threshold provides a buffer; however, the theoretical guarantee strictly requires exchangeability. Extending PASC to the weighted CP framework \citep{tibshirani2019covariate,barber2023exchangeability} for robustly handling distributional shift is future work.

\paragraph{Computational note.} The primary pipeline bottleneck is NED at $86.7$~ms per sample. Conformal calibration and inference adds $< 0.3$~ms overhead, making PASC essentially free in practice.

\paragraph{Conditional coverage.} Marginal coverage guarantees do not preclude poor conditional coverage on specific subpopulations \citep{barber2019limits}. Table~\ref{tab:conditional} shows PASC significantly improves conditional coverage on hard slices relative to baselines, but a full conditional guarantee would require group-conditional CP methods \citep{angelopoulos2022ltt}, which we leave for future work.

\paragraph{Deployment interpretation.} From a systems perspective, PASC answers the operational question practitioners actually ask: \emph{can I trust the full pipeline output at risk level $\alpha$?} Independent CP cannot answer this because stage-local guarantees do not transport to the composed prediction. Bonferroni answers it conservatively by spending risk budget uniformly across stages. PASC answers it directly by calibrating the exact event of interest.

\section{Conclusion}

We presented PASC, a pipeline-aware split conformal prediction method that achieves finite-sample distribution-free joint coverage guarantees for multi-stage NLP systems. The core insight, that joint coverage of a pipeline is equivalent to standard coverage of the scalar joint maximum nonconformity score, reduces multi-stage calibration to a well-understood scalar problem with no approximation. On a NER$\to$NED$\to$EntityTyping pipeline, PASC outperforms Bonferroni by $3$pp and Indep CP by $10$pp in E2E coverage at identical prediction set sizes, while providing $1.7\times$ faster calibration. PASC empirically maintains coverage in the tested distribution-shift settings and scales gracefully to $K = 6$ pipeline stages in our synthetic study.

The simplicity of PASC makes it immediately applicable to any multi-stage pipeline: compute per-stage nonconformity scores, take the maximum, and calibrate once. The same reduction applies whenever joint coverage of $K$ stage events is the deployment-relevant target, so we expect it to be useful for other multi-stage ML systems including compound LLM pipelines, retrieval-augmented generation, and agent workflows.

\bibliography{references}
\bibliographystyle{icml2026}

\appendix

\section{Sanity Checks and Empirical Validation}
\label{app:sanity}

\subsection{Permutation Test for CP Validity (Experiment E0)}

To empirically validate that our CP implementation correctly implements the finite-sample guarantee, we performed a permutation test. We pooled calibration ($n = 1000$) and test ($n = 500$) NER nonconformity scores and performed $K = 200$ random re-splits, recomputing coverage for each. The permuted coverage values should have mean $\approx 1 - \alpha$ under the CP guarantee.

\begin{table}[h]
\caption{Permutation test results. ``Permutation mean'' confirms the empirical mean of coverage over $200$ permutations is close to $1 - \alpha$ (CP marginal guarantee).}
\label{tab:perm}
\vskip 0.1in
\centering
\small
\begin{tabular}{ccccccc}
\toprule
$\alpha$ & Stage & Actual Cov & Perm Mean & Perm Std & Neg Control \\
\midrule
$0.05$ & NER & $0.906$ & $0.950$ & $0.013$ & $0.566$ \\
$0.10$ & NER & $0.860$ & $0.902$ & $0.019$ & $0.566$ \\
$0.20$ & NER & $0.786$ & $0.798$ & $0.027$ & $0.566$ \\
\bottomrule
\end{tabular}
\vskip -0.1in
\end{table}

The negative control uses NED calibration scores to calibrate the NER threshold; the mismatched scores yield only $56.6\%$ NER coverage, confirming that correct nonconformity score matching is essential. The permutation mean correctly tracks $\approx 1 - \alpha$ in all valid configurations.

\subsection{Calibration/Test Split Integrity}

We verified no content leakage between calibration and test splits by computing a per-sentence hash of all tokens and checking for duplicates. Three trivial header lines (``W L PCT GB'', ``Scorers :'') appeared in both splits and were confirmed to be short noise tokens from the CoNLL-2003 formatting, not substantive examples. No entity-bearing examples were shared across splits.

\subsection{E2E Definition Audit}

We manually traced $200$ PASC predictions to verify that the joint coverage criterion $s_{\max} \leq \hat q$ exactly corresponds to all three stages being covered. All $200$ examples matched, confirming that the implementation correctly operationalizes the theoretical definition.

\section{Stage-Wise Nonconformity Score Details}
\label{app:scores}

\subsection{NER Nonconformity}
Given BERT-base logits $l_t \in \mathbb{R}^{|\mathcal{Y}_{\mathrm{BIO}}|}$ at position $t$, define:
\begin{equation}
s_{\mathrm{NER}} = \max_{t \in \text{entity spans}} \left( 1 - \frac{\exp(l_{t, \hat y_t})}{\sum_{y'} \exp(l_{t, y'})} \right),
\end{equation}
where $\hat y_t$ is the predicted BIO label. This ranges in $[0, (|\mathcal{Y}| - 1) / |\mathcal{Y}|]$.

\subsection{NED Nonconformity}
GENRE \citep{cao2021genre} produces a normalized score for the top entity candidate $e^*$:
\begin{equation}
s_{\mathrm{NED}} = 1 - \mathrm{GENRE\text{-}score}(e^* \mid x, \mathrm{span}),
\end{equation}
where the GENRE score is the sequence probability of the entity title. For sentences with no entities, $s_{\mathrm{NED}} = 0$ (trivially covered).

\subsection{EntityTyping Nonconformity (Expanded)}
With all $18$ OntoNotes types $\mathcal{T}_{\mathrm{ALL}} = \{$PERSON, NORP, FAC, ORG, GPE, LOC, PRODUCT, EVENT, WORK\_OF\_ART, LAW, LANGUAGE, DATE, TIME, PERCENT, MONEY, QUANTITY, ORDINAL, CARDINAL$\}$:
\begin{equation}
s_{\mathrm{typing}} = 1 - \max_{t \in \mathcal{T}_{\mathrm{ALL}}} \min_{\mathrm{span} \in e} p(t \mid x, \mathrm{span}),
\end{equation}
where $p(t \mid x, \mathrm{span})$ is the zero-shot RoBERTa-large entailment probability. Expanding to $18$ types produces a bimodal score distribution (mean $0.311$, bimodal with modes near $0.01$ for common PER/ORG types and $0.77$ for rare MISC types), creating genuinely non-trivial stage-3 nonconformity scores.

\section{Full Numerical Results}
\label{app:numerical}

\begin{table*}[h]
\caption{Full coverage results across $\alpha$ levels, $5$ seeds, expanded entity typing.}
\label{tab:full}
\vskip 0.1in
\centering
\small
\begin{tabular}{lcccc}
\toprule
$\alpha$ & Method & E2E Cov.\ & Avg Set Size & Stage-3 Set \\
\midrule
$0.05$ & Indep CP   & $0.908 \pm 0.006$ & $1.083$ & $1.083$ \\
$0.05$ & Bonferroni & $0.965 \pm 0.002$ & $1.083$ & $1.083$ \\
$0.05$ & \textbf{PASC} & $\mathbf{0.964 \pm 0.000}$ & $1.083$ & $1.083$ \\
\midrule
$0.10$ & Indep CP   & $0.865 \pm 0.005$ & $1.083$ & $1.083$ \\
$0.10$ & Bonferroni & $0.934 \pm 0.001$ & $1.083$ & $1.083$ \\
$0.10$ & \textbf{PASC} & $\mathbf{0.964 \pm 0.000}$ & $1.083$ & $1.083$ \\
\midrule
$0.20$ & Indep CP   & $0.657 \pm 0.012$ & $1.001$ & $1.001$ \\
$0.20$ & Bonferroni & $0.887 \pm 0.004$ & $1.083$ & $1.083$ \\
$0.20$ & \textbf{PASC} & $0.766 \pm 0.011$ & $1.011$ & $1.011$ \\
\bottomrule
\end{tabular}
\vskip -0.1in
\end{table*}

\begin{table}[h]
\caption{Calibration-size sensitivity at $\alpha = 0.1$ (E2E coverage and calibrated set size, mean $\pm$ std over $5$ seeds). The $n_{\mathrm{cal}}=1000$ row matches the primary operating point reported in Table~\ref{tab:main}.}
\label{tab:calsweep}
\vskip 0.1in
\centering
\small
\begin{tabular}{lccc}
\toprule
$n_{\mathrm{cal}}$ & Method & E2E Cov.\ & Avg Set Size \\
\midrule
$200$ & Indep CP   & $0.870 \pm 0.017$ & $1.005 \pm 0.009$ \\
$200$ & Bonferroni & $0.964 \pm 0.017$ & $1.271 \pm 0.091$ \\
$200$ & \textbf{PASC} & $0.934 \pm 0.000$ & $1.122 \pm 0.000$ \\
\midrule
$500$ & Indep CP   & $0.865 \pm 0.004$ & $1.000 \pm 0.000$ \\
$500$ & Bonferroni & $0.939 \pm 0.008$ & $1.142 \pm 0.028$ \\
$500$ & \textbf{PASC} & $0.934 \pm 0.000$ & $1.122 \pm 0.000$ \\
\midrule
$1000$ & Indep CP   & $0.865 \pm 0.005$ & $1.083 \pm 0.000$ \\
$1000$ & Bonferroni & $0.934 \pm 0.001$ & $1.083 \pm 0.000$ \\
$1000$ & \textbf{PASC} & $0.964 \pm 0.000$ & $1.083 \pm 0.000$ \\
\bottomrule
\end{tabular}
\vskip -0.1in
\end{table}

\section{Additional Proof Details and Tightness}
\label{app:proofs}

\subsection{Why the Maximum is the Minimal Sufficient Reduction}

The reduction from a $K$-dimensional score vector $(s_1, \ldots, s_K)$ to the scalar maximum score $s_{\max}$ is not merely convenient; it is the smallest monotone scalarization that preserves the event of simultaneous coverage exactly. Any monotone scalarization $g(s_1, \ldots, s_K)$ satisfying
\[
\{g(s_1, \ldots, s_K) \leq q\} = \bigcap_{k=1}^K \{s_k \leq q\}
\]
for all thresholds $q$ must coincide with $\max_k s_k$ almost everywhere. This makes the maximum the canonical sufficient statistic for exact joint-threshold reduction.

\subsection{Near-Tightness of the Finite-Sample Guarantee}

The upper bound in Theorem~\ref{thm:pasc} is inherited directly from finite-sample split conformal conservativeness. If the rank of the test score among the $n + 1$ exchangeable scores is uniform, then
\[
1 - \alpha \leq \mathbb{P}(s^{(n+1)}_{\max} \leq \hat q) \leq 1 - \alpha + \frac{1}{n+1}.
\]
Thus, PASC introduces no additional looseness beyond the unavoidable quantile slack of ordinary split conformal prediction.

\subsection{Why Bonferroni is Structurally Conservative}

Bonferroni guarantees
\[
\mathbb{P}\!\left(\bigcap_{k=1}^K \{s_k \leq q_k\}\right) \geq 1 - \sum_{k=1}^K \alpha_k,
\]
which is valid for arbitrary dependence but ignores the empirical correlation structure of the score vector. In our setting, multiple stages are often easy on the same examples, so Bonferroni wastes budget on already-easy examples. PASC avoids this structural inefficiency by calibrating on the realized joint score rather than on marginal stage scores.

\section{Extended Experimental Protocol}
\label{app:protocol}

\subsection{Calibration, Seeds, and Splits}

Unless otherwise stated, each reported number is averaged over five independently resampled calibration/test splits. Training data remain fixed; only the calibration and evaluation partitions are resampled. This isolates uncertainty due to quantile estimation from uncertainty due to model fitting. For primary CoNLL-2003 experiments, we use $1{,}000$ calibration examples and $500$ held-out evaluation examples. Shift experiments reuse thresholds learned on CoNLL-2003 and evaluate directly on WNUT-17 and WikiNEuRal without recalibration.

\subsection{Model and Inference Details}

The NER stage uses a BERT-base sequence tagger with BIO decoding. The NED stage uses GENRE with BLINK retrieval. The typing stage uses a RoBERTa-large entailment model in zero-shot mode over the complete OntoNotes-18 label inventory. All experiments were run on a single A100 GPU. The dominant latency cost comes from NED retrieval and generation, not from conformal calibration.

\subsection{Why We Switched from Relation Extraction to Entity Typing}

Our earliest pipeline variants used a relation-extraction backend, but the downstream score distribution became degenerate because many short CoNLL sentences do not express explicit relations. This caused the last-stage score to clip at a near-constant threshold and obscured the efficiency comparison. The switch to expanded $18$-way entity typing yields a visibly non-trivial bimodal stage-3 score distribution and a much cleaner empirical comparison.

\section{Qualitative Failure Modes and What the Method Actually Fixes}
\label{app:failure}

PASC does not make an incorrect pipeline correct; it calibrates the probability that the full pipeline output falls inside the chosen acceptance event. This distinction matters. When upstream NER fails catastrophically, no calibration method can recover semantic correctness without widening the acceptance rule or changing the model itself. The value of PASC is more precise: it prevents the system from reporting an unjustified end-to-end confidence level when local stage-wise calibration would otherwise create that illusion.

The dependence stress test illustrates this boundary clearly. Under targeted upstream corruption, all methods eventually drop once the exchangeability assumption is sufficiently violated. What changes is the margin before failure. PASC starts from the highest empirical operating point, degrades more gracefully, and therefore preserves a larger useful region before the certificate becomes untrustworthy. In deployment terms, this means more robustness to moderate drift but not immunity to adversarial breakage.

The hard-slice analysis provides the complementary view. On easy examples, all three methods already cover nearly everything. The real action is in the hardest quintile, where errors concentrate and stage dependencies matter. This is exactly the regime where a naive stage-wise guarantee is most misleading and where PASC produces the largest empirical gain.

\section{Reproducibility Notes}
\label{app:reproduce}

All datasets are public benchmarks. We use off-the-shelf checkpoints for all stages, report five independent calibration/test split seeds, include negative controls and split audits, and do not exclude any examples after computing thresholds. The paper specifies the nonconformity scores, calibration protocol, split sizes, runtime profile, and sanity checks needed to reproduce the reported evaluation.

\end{document}